\documentclass[letterpaper, 10 pt, conference]{ieeeconf}

\IEEEoverridecommandlockouts

\usepackage{hyperref}
\usepackage{ntheorem}
\usepackage{booktabs} 
\usepackage{adjustbox} 
\usepackage{cite}
\usepackage{amsmath,amssymb,amsfonts}
\usepackage{algorithmic}
\usepackage{graphicx}
\usepackage{textcomp}
\usepackage{xcolor}
\newtheorem{theorem}{Theorem} 
\newcommand{\norm}[2]{\left \lVert #1 \right \rVert_{#2}}
\newcommand{\revision}[1]{\textcolor{black}{#1}}

\title{\LARGE \bf
A Mixed-Integer Conic Program for the Moving-Target Traveling Salesman Problem based on a Graph of Convex Sets
}

\author{Allen George Philip$^{1}$, Zhongqiang Ren$^{2}$, Sivakumar Rathinam$^{1}$ and Howie Choset$^{3}$
    \thanks{$^{1}$Allen George Philip and Sivakumar Rathinam are with the Department of Mechanical Engineering, Texas A\&M University,
		College Station, TX 77843-3123.
		Emails: {\tt \{y262u297, srathinam\}@tamu.edu}}%
    \thanks{$^{2}$Zhongqiang Ren is with Shanghai Jiao Tong University, 800 Dongchuan Road, Shanghai, China. Email: {\tt zhongqiang.ren@sjtu.edu.cn}}
    \thanks{$^{3}$Howie Choset is with Carnegie Mellon University, 5000 Forbes Ave., Pittsburgh, PA 15213, USA. Email: {\tt choset@andrew.cmu.edu}}%
    \thanks{© 2024 IEEE. Personal use of this material is permitted.  Permission from IEEE must be obtained for all other uses, in any current or future media, including reprinting/republishing this material for advertising or promotional purposes, creating new collective works, for resale or redistribution to servers or lists, or reuse of any copyrighted component of this work in other works. Final published article in IEEE Xplore is now available at \url{https://ieeexplore.ieee.org/document/10802374}}
}

\begin{document}

\maketitle
\thispagestyle{empty}
\pagestyle{empty}

\begin{abstract}
This paper introduces a new formulation that finds the optimum for the Moving-Target Traveling Salesman Problem (MT-TSP), which seeks to find a shortest path for an agent, that starts at a depot, visits a set of moving targets exactly once within their assigned time-windows, and returns to the depot. The formulation relies on the key idea that when the targets move along lines, their trajectories become convex sets within the space-time coordinate system. The problem then reduces to finding the shortest path within a graph of convex sets, subject to some speed constraints. We compare our formulation with the current state-of-the-art Mixed Integer Conic Program (MICP) formulation for the MT-TSP. The experimental results show that our formulation outperforms the MICP for instances with up to 20 targets, with up to two orders of magnitude reduction in runtime, and up to a 60\% tighter optimality gap. We also show that the solution cost from the convex relaxation of our formulation provides significantly tighter lower-bounds for the MT-TSP than the ones from the MICP.
\end{abstract}

\section{Introduction}

Given a set of stationary targets and the cost of traversal between any pair of these targets, the classical Traveling Salesman Problem (TSP) seeks to find the shortest tour for an agent such that it visits all the targets exactly once. The TSP is one of \revision{the most} fundamental problems in combinatorial optimization, with several applications including unmanned vehicle planning \cite{oberlin2010today,liu2018efficient,ryan1998reactive, yu2002implementation}, transportation and delivery \cite{ham2018integrated}, monitoring and surveillance \cite{venkatachalam2018two,saleh2004design}, disaster management \cite{cheikhrouhou2020cloud}, precision agriculture \cite{conesa2016mix}, and search and rescue \cite{zhao2015heuristic, brumitt1996dynamic}. In this paper, we consider the generalization of the TSP, where the targets follow some predefined trajectories, and also have associated time-windows during which they need to be visited. The objective is to minimize the distance traversed by the agent. We refer to this generalization as the Moving-Target TSP or MT-TSP for short. In the literature, we find different variants of the MT-TSP, motivated by practical applications such as defending an area from oncoming hostile rockets or Unmanned Aerial \revision{Vehicles} \cite{helvig2003,smith2021assessment,stieber2022}, monitoring and surveillance \cite{deMoraes2019,wang2023moving,marlow2007travelling,maskooki2023bi}, resupply missions with moving targets \cite{helvig2003}, dynamic target tracking \cite{englot2013efficient}, and industrial robot planning \cite{chalasani1999approximating}.

\vspace{1mm}
The speed of the targets are generally assumed to be no greater than the agent's maximum speed \cite{helvig2003}. When the speed of all the targets reduces to 0, the MT-TSP reduces to the classical TSP. Hence, MT-TSP is NP-hard. Currently, the literature presents exact and approximation algorithms for some very restricted cases of the MT-TSP variants where the targets move in the same direction with the same speed \cite{chalasani1999approximating,hammar1999}, move along the same line \cite{helvig2003,hassoun2020}, or move along lines through the depot, towards or away from it \cite{helvig2003}. Several heuristic based approaches have also been introduced in the literature \cite{bourjolly2006orbit,choubey2013,deMoraes2019,englot2013efficient,groba2015solving,jiang2005tracking,marlow2007travelling,ucar2019meta,wang2023moving} that finds feasible solutions, but gives no information on how far they are from the optimum. 

\begin{figure}[t]
    \centering
    \includegraphics[width=\linewidth]{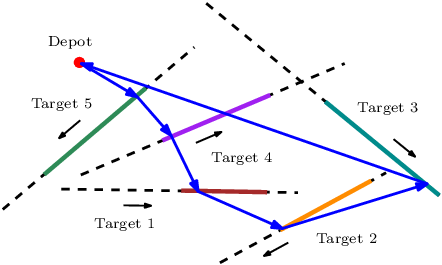}
    \caption{A feasible solution to an example instance of the MT-TSP where 5 targets move along lines. The agent's tour is given in blue, and the part of each target's trajectory corresponding to its time-window where they can be visited by the agent are given by colored solid segments.}
    \label{fig:feas_tour}
\end{figure}

\vspace{1mm}
The objective of this paper is to find exact solutions to a less restricted case of the MT-TSP where each target moves along its own line, with fixed speeds (Fig.~\ref{fig:feas_tour}). Currently, the only approach that does this, is the MICP (specifically, a mixed-integer Second Order Conic Program (SOCP)) introduced in \cite{stieber2022}. Hence, we use this formulation as a baseline, and introduce an alternative formulation that finds the optimum for the MT-TSP. Our formulation relies on the key idea that when the targets move along lines, their trajectories become convex sets within the space-time coordinate system. This reduces the MT-TSP to a problem of finding a shortest tour in a graph of convex sets \cite{marcucci2024shortest}, subject to some additional speed constraints. This allows us to formulate the MT-TSP as a biconvex, binary program, which we then reformulate as a mixed-integer SOCP, by leveraging the ideas presented in \cite{marcucci2024shortest}. 

\vspace{1mm}
We prove that our approach finds the optimum to the MT-TSP, and provide computational results to corroborate the performance of our formulation. We find that our approach vastly outperforms the baseline and scales much better when increasing the time-window duration and the number of targets, achieving up to two orders of magnitude faster average runtime and up to 60\% improvement in the average optimality gap. We also show that our formulation has a much stronger convex relaxation than the baseline, which can be used to find lower-bounds to the MT-TSP with significantly lower computational burden.

\section{Problem Definition}
All the targets and the agent move in a 2D $(x,y)$ plane. Let $V_{tar} := \{1,2,\cdots,n\}$ denote the set of $n$ moving targets, and let $s$ be the depot. Without loss of generality, we make a copy of the depot and refer to it as $s'$ and require the agent to return to $s'$ at the end of its tour. Let $V := V_{tar} \cup \{s,s'\}$. Given a node $i \in V$, the time-window associated with it is denoted by $[\underline{t}_i, \overline{t}_i]$. The $(x,y)$ position occupied by node $i$ at time $\underline{t}_i$ and $\overline{t}_i$ is denoted by $(\underline{p}_{i,x},\underline{p}_{i,y})$ and $(\overline{p}_{i,x},\overline{p}_{i,y})$ respectively, and the velocity coordinates of node $i$ is denoted by $(v_{i,x},v_{i,y})$. The maximum agent speed is denoted by $v_{max}$. Note that we fix $\underline{t}_s = \overline{t}_s = 0$ since the agent tour starts at time $0$. In addition, we fix $\underline{t}_{s'}=0$ and $\overline{t}_{s'}=T$ where $T$ is the time-horizon, so that the agent is free to complete the tour at any time within $[0,T]$. Also note that the velocity of the depot, $(v_{s,x},v_{s,y})$ and the velocity of the depot's copy, $(v_{s',x},v_{s',y})$ are fixed to be $(0,0)$ since they are stationary. We say that the agent visits a moving target (say $i$) if there is a time instant in $[\underline{t}_i, \overline{t}_i]$ when the position of the agent coincides with the position of the target $i$. Any feasible tour for the agent will start from $s$, visit each target in $V_{tar}$ exactly once and return to $s'$.  The objective of the MT-TSP is to find a feasible tour for the agent such that the distance traveled by the agent along the tour is minimized.

\section{MICP for MT-TSP}

This section presents the current state-of-the-art MICP formulation introduced in \cite{stieber2022} for the MT-TSP. The formulation is slightly modified in this paper to include the additional requirement that the agent tour ends at $s'$. To proceed further, we first construct a directed graph $(V,E)$ where the edges in $E$ are added as follows: from node $s$ to all the nodes in $V_{tar}$, from each node in $V_{tar}$ to every other node in $V_{tar}$, and finally, from every node in $V_{tar}$, to node $s'$. For any node $i \in V$, $E^{in}_{i}$ and $E^{out}_{i}$ denotes the set of all edges entering and exiting $i$. 

Next, we define the decision variables for this formulation. For each node $i \in V$, the real variable $t_i$ represents the time at which the agent visits $i$. For each edge $e=(i,j) \in E$, variable $y_e \in [0,1]$ represents the flow through that edge. In a binary program, $y_e \in \{0,1\}$, and it represents the decision of whether or not edge $e$ is chosen. We will consider $y_e$ as a binary variable unless otherwise stated. The auxiliary variable $\Tilde{l}(p_i,p_j) \geq 0$ for each $e=(i,j) \in E$ represents $y_e\norm{p_j-p_i}{2}$ where, for some node $i$, $p_i=(p_{i,x},p_{i,y})$ describes the position of that node at time $t_i$.
The formulation also introduces for each $e=(i,j) \in E$, other real auxiliary variables, $l_x(p_i,p_j)$ and $l_y(p_i,p_j)$ which \revision{represent} $p_{j,x}-p_{i,x}$ and $p_{j,y}-p_{i,y}$ respectively, and finally $\overline{l}(p_i,p_j) \geq 0$ which will be used to define the second-order cone constraints. 

We also note that there is a parameter $R$ which denotes the length of the diagonal of the square area that contains the depot and all the target trajectories. The fact that the Euclidean length of any line segment within the square area cannot exceed $R$ will be used in formulating one of the constraints in the formulation. Now, the MICP formulation for the MT-TSP is as follows:

\begin{align}
    \label{eq:SOCPobj}
    &\min \sum_{e=(i,j)\in E} \Tilde{l}(p_i,p_j) \\
    \intertext{subject to constraints}
    \label{eq:depotFlowOut}
    &\sum_{e \in E^{out}_s}y_e = 1, \\
    \label{eq:depotFlowIn}
    &\sum_{e \in E^{in}_{s'}}y_e = 1, \\
    \label{eq:targetFlowIn}
    &\sum_{e \in E^{in}_i}y_e = 1, \;\; \forall \; i \in V_{tar}, \\
    \label{eq:flowConservation}
    &\sum_{e \in E^{in}_i}y_e = \sum_{e \in E^{out}_i}y_e, \;\; \forall \; i \in V_{tar}, \\
    \label{eq:nodeTimeWin}
    & \underline{t}_i \leq t_i \leq \overline{t}_i, \;\; \forall \; i \in V, \\
    \label{eq:defSOCPlx}
    \begin{split}
        &l_x(p_i,p_j)-((\underline{p}_{j,x}+t_jv_{j,x}-\underline{t}_jv_{j,x})\\
        &-(\underline{p}_{i,x}+t_iv_{i,x}-\underline{t}_iv_{i,x})) = 0, \;\; \forall \; e=(i,j) \in E,
    \end{split} \\
    \label{eq:defSOCPly}
    \begin{split}
        &l_y(p_i,p_j)-((\underline{p}_{j,y}+t_jv_{j,y}-\underline{t}_jv_{j,y})\\
        &-(\underline{p}_{i,y}+t_iv_{i,y}-\underline{t}_iv_{i,y}))=0, \;\; \forall \; e=(i,j) \in E,
    \end{split} \\
    \label{eq:speedSOCPBigM}
    &\Tilde{l}(p_i,p_j) \leq v_{max}(t_j-t_i+T(1-y_e)), \;\; \forall \; e=(i,j) \in E, \\
    \label{eq:coneSOCPBigM}
    &\overline{l}(p_i,p_j) = \Tilde{l}(p_i,p_j) + R(1-y_e), \;\; \forall \; e=(i,j) \in E, \\
    \label{eq:normSquaredSOCP}
    \begin{split}
        &(l_x(p_i,p_j))^2+(l_y(p_i,p_j))^2 \leq (\overline{l}(p_i,p_j))^2, \\
        & \forall \; e=(i,j) \in E.
    \end{split}
\end{align}

The objective \eqref{eq:SOCPobj} is to minimize the total tour length of the agent. The condition that the agent departs from the depot once, and arrives at the depot's copy once, is described by \eqref{eq:depotFlowOut} and \eqref{eq:depotFlowIn} respectively. The constraints, \eqref{eq:targetFlowIn} \revision{ensure} that each target is visited exactly once by the agent, and the flow conservation for all the target nodes are ensured by \eqref{eq:flowConservation}. Constraints \eqref{eq:depotFlowOut} to \eqref{eq:flowConservation} are fundamental to the MT-TSP, and \revision{ensure} a valid agent path that starts at the depot, visits all the targets exactly once, and returns to the depot. Hence, these constraints will be repeated for all the formulations in this article. The condition requiring the agent to visit each node within its time-window is given by \eqref{eq:nodeTimeWin}, and the definitions of the auxiliary variables $l_{x}(p_i,p_j)$ and $l_{y}(p_i,p_j)$ are captured through \eqref{eq:defSOCPlx} and \eqref{eq:defSOCPly}, respectively.

\vspace{1mm}
Now, we will explain the big-$M$ constraints, \eqref{eq:speedSOCPBigM}, and \eqref{eq:coneSOCPBigM}, for each edge $e=(i,j) \in E$. First, consider the time-feasibility constraints, \eqref{eq:speedSOCPBigM}. These constraints \revision{describe} the condition that if $y_e=1$, then $\Tilde{l}(p_i,p_j) \leq v_{max}(t_j-t_i)$. However, if $y_e=0$, then no restrictions are placed on $t_i$ and $t_j$. Second, consider the constraints, \eqref{eq:coneSOCPBigM}. With the help of \eqref{eq:defSOCPlx} and \eqref{eq:defSOCPly}, these constraints, along with the second-order cone constraints, \eqref{eq:normSquaredSOCP} \revision{describe} the condition that $\Tilde{l}(p_i,p_j) \geq \norm{p_j-p_i}{2}$ if $y_e=1$. However, if $y_e=0$, then $\Tilde{l}(p_i,p_j)$ is free to take any value.

\vspace{1mm}
Although this formulation describes the MT-TSP well, it is challenging to solve in practice. 
In this paper, we present a graph of convex sets (GCS) based MICP (MICP-GCS) that is significantly faster to solve, and provides much stronger relaxations. Prior to presenting this formulation, we will restate the current MICP as a biconvex binary program. This will aid us in proving that an optimal solution to MICP-GCS indeed provides an optimal solution to the MT-TSP.

\section{MICP on the Graph of Convex Sets}

\subsection{Biconvex Binary Program for the MT-TSP}
In this section, we will restate the MICP for the MT-TSP as a biconvex, binary program. For simplicity, we will refer to this program as the biconvex formulation. First, we present the decision variables. For each node $i \in V$, we reuse the variable $t_i$ from the MICP. In addition, we introduce real auxiliary variables, $p_{i,x}$ and $p_{i,y}$ that explicitly \revision{define} the $(x,y)$ coordinates of $p_i$. For each edge $e=(i,j) \in E$, we use the binary variable $y_e$ from before, as well as introduce real variables $z_{e,x}$, $z_{e,y}$, and $z_{e,t}$, representing $y_ep_{i,x}$, $y_ep_{i,y}$, and $y_et_i$, and real variables $z'_{e,x}$, $z'_{e,y}$, and $z'_{e,t}$, representing $y_ep_{j,x}$, $y_ep_{j,y}$, and $y_et_j$. Finally, we replace the variables 
$l_x(p_i,p_j)$, $l_y(p_i,p_j)$, and $\Tilde{l}(p_i,p_j)$ from the MICP with real auxiliary variables $l_x(\Tilde{z}_e, \Tilde{z}'_e)$, $l_y(\Tilde{z}_e, \Tilde{z}'_e)$, and $l(\Tilde{z}_e, \Tilde{z}'_e) \geq 0$ respectively. Note that $\Tilde{z}_e$ and $\Tilde{z}'_e$ are notations describing $(z_{e,x},z_{e,y})$ and $(z'_{e,x},z'_{e,y})$ respectively. Also, notations $z_e$, and $z'_e$ \revision{describe} $(\Tilde{z}_e,z_{e,t})$, and $(\Tilde{z}'_e,z'_{e,t})$ respectively. 
The variable $\overline{l}(p_i,p_j)$ introduced previously for each edge is not used in the biconvex formulation. This is because the big-$M$ constraints are removed here. The formulation is presented below:

\begin{align}
    \label{eq:objBiconvex}
    &\min \sum_{e=(i,j)\in E}l(\Tilde{z}_e,\Tilde{z}'_e) \\
    \intertext{subject to constraints \eqref{eq:depotFlowOut}, \eqref{eq:depotFlowIn}, \eqref{eq:targetFlowIn}, \eqref{eq:flowConservation}, \eqref{eq:nodeTimeWin},}
    \label{eq:defPx}
    &p_{i,x} = \underline{p}_{i,x} + t_{i}v_{i,x} - \underline{t}_{i}v_{i,x}, \;\; \forall \; i \in V, \\
    \label{eq:defPy}
    &p_{i,y} = \underline{p}_{i,y} + t_iv_{i,y} - \underline{t}_iv_{i,y}, \;\; \forall \; i \in V, \\
    \label{eq:lxBiconvex}
    &l_x(\Tilde{z}_e,\Tilde{z}'_e) = (z'_{e,x}-z_{e,x}), \;\; \forall e=(i,j)\in E, \\
    \label{eq:lyBiconvex}
    &l_y(\Tilde{z}_e,\Tilde{z}'_e) = (z'_{e,y}-z_{e,y}), \;\; \forall e=(i,j)\in E, \\
    \label{eq:speedBiconvex}
    &l(\Tilde{z}_e,\Tilde{z}'_e) \leq v_{max}(z'_{e,t}-z_{e,t}), \;\; \forall e=(i,j) \in E, 
    \end{align}
    
    \begin{align}
    \label{eq:coneBiconvex}
    \begin{split}
        &(l_x(\Tilde{z}_e,\Tilde{z}'_e))^2+(l_y(\Tilde{z}_e,\Tilde{z}'_e))^2 \leq (l(\Tilde{z}_e,\Tilde{z}'_e))^2, \\
        &\forall \; e=(i,j) \in E,
    \end{split} \\
    \label{eq:defz,z'}
    &z_e=(y_ep_i,y_et_i), \; z'_e=(y_e p_j,y_et_j), \;\; \forall e=(i,j) \in E.
\end{align}

This formulation shares constraints \eqref{eq:depotFlowOut} to \eqref{eq:nodeTimeWin} from the MICP. Constraints \eqref{eq:defPx} and \eqref{eq:defPy} \revision{describe} variables $p_{i,x}$ and $p_{i,y}$. Apart from the binary requirement for all the $y_e$ variables, the non-convexities of this program \revision{come} only from the bilinear constraints, \eqref{eq:defz,z'}. These constraints will be utilized to achieve the role satisfied by the big-$M$ constraints in the MICP. First, notice how for each edge $e=(i,j) \in E$, when $y_e=0$, \eqref{eq:defz,z'} becomes $z_e=z'_e=(0,0,0)$. Consequently, \eqref{eq:lxBiconvex}, \eqref{eq:lyBiconvex} \revision{become} $l_x(\Tilde{z}_e,\Tilde{z}'_e)=l_y(\Tilde{z}_e,\Tilde{z}'_e)=0$. However, when $y_e=1$, \eqref{eq:defz,z'} becomes $z_e=(p_i,t_i)$, $z'_e=(p_j,t_j)$ and consequently, \eqref{eq:lxBiconvex}, \eqref{eq:lyBiconvex} \revision{become} $l_x(\Tilde{z}_e,\Tilde{z}'_e)=p_{j,x}-p_{i,x}$, and $l_y(\Tilde{z}_e,\Tilde{z}'_e)=p_{j,y}-p_{i,y}$. 

Now, consider \eqref{eq:speedBiconvex}. For each edge $e=(i,j) \in E$, when $y_e=1$, these constraints become $l(\Tilde{z}_e,\Tilde{z}'_e) \leq v_{max}(t_j-t_i)$. However, when $y_e=0$, we get $l(\Tilde{z}_e,\Tilde{z}'_e) \leq 0$, allowing $t_i$, and $t_j$ to be free. Finally, consider \eqref{eq:coneBiconvex}. We see how these constraints with the help of \eqref{eq:lxBiconvex}, \eqref{eq:lyBiconvex} \revision{become} $l(\Tilde{z}_e,\Tilde{z}'_e) \geq \norm{p_j-p_i}{2}$ when $y_e=1$, but \revision{allow} $l(\Tilde{z}_e,\Tilde{z}'_e)$ to take any value when $y_e=0$. Recall how the big-$M$ constraints, \eqref{eq:speedSOCPBigM}, \eqref{eq:coneSOCPBigM}, and the constraints, \eqref{eq:normSquaredSOCP} \revision{achieve} the same role as \eqref{eq:speedBiconvex} and \eqref{eq:coneBiconvex} for $\Tilde{l}(p_i,p_j)$. Hence, in the biconvex formulation, constraints \eqref{eq:speedBiconvex} and \eqref{eq:coneBiconvex} \revision{replace} constraints \eqref{eq:speedSOCPBigM}, \eqref{eq:coneSOCPBigM} and \eqref{eq:normSquaredSOCP} from the MICP, while establishing the relationship,

\begin{align}
    \label{eq:relationlzlp}
    &l(\Tilde{z}_e,\Tilde{z}'_e)=\Tilde{l}(p_i,p_j).
\end{align}

From \eqref{eq:relationlzlp}, we see how both the MICP and the biconvex formulation have the same optimal value. Moreover, we can recover an optimal tour for the MT-TSP from the solution of the biconvex formulation by recovering the $(p_i,t_i)$ for each node $i \in V$.

\begin{figure}[t]
    \centering
    \includegraphics[width=\linewidth]{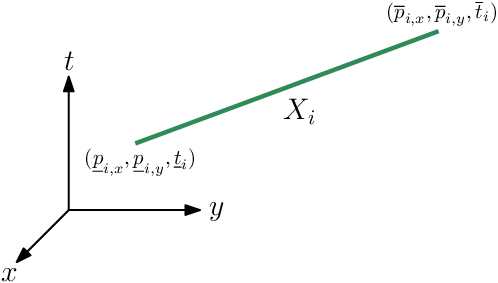}
    \caption{The trajectory-segment that corresponds to the time-window of some node $i \in V$, is a line segment within the space-time coordinate system $(x,y,t)$. The set of all points in the line segment forms the convex set $X_i$.}
    \label{fig:nodeConvex}
\end{figure}

\vspace{1mm}
Now, we discuss the key idea behind the GCS-based MICP we introduce in this paper. Notice how the set of all $(p_{i,x},p_{i,y},t_i)$ that satisfy \eqref{eq:nodeTimeWin}, \eqref{eq:defPx}, and \eqref{eq:defPy} for each node $i \in V$ represents the trajectory-segment of that node within its time-window. These trajectory-segments are line segments when expressed within the $(x,y,t)$ coordinate system as shown in the example illustration in Fig.~\ref{fig:nodeConvex}. Hence, the trajectory-segment corresponding to each node $i$ can be considered a convex set $X_i$ corresponding to that node, and the agent is required to visit a point within $X_i$. This allows us to solve the MT-TSP by leveraging the ideas presented in \cite{marcucci2024shortest}, where the problem was to find the shortest paths in graphs of convex sets. To summarize, \eqref{eq:nodeTimeWin}, \eqref{eq:defPx}, and \eqref{eq:defPy}, together can be represented using the following set of constraints:

\begin{align}
    \label{eq:trajConvex}
    &(p_i,t_i) \in X_i, \;\; \forall \; i \in V.
\end{align}

Note that the non-convexities that arise from the bilinear constraints make the biconvex formulation very challenging to solve. However, in the next section, we will introduce our MICP-GCS formulation for the MT-TSP, which can be easily handled by standard solvers.

\subsection{GCS-Based Mixed Integer Conic Program (MICP-GCS)}
In this section, we present our new MICP-GCS formulation for the MT-TSP. First, we discuss the decision variables for this formulation. For each edge $e=(i,j) \in E$, we use the same variables $y_e$, $z_{e,x}$, $z_{e,y}$, $z_{e,t}$, $z'_{e,x}$, $z'_{e,y}$, $z'_{e,t}$, $l_x(\Tilde{z}_e, \Tilde{z}'_e)$, $l_y(\Tilde{z}_e, \Tilde{z}'_e)$, $l(\Tilde{z}_e, \Tilde{z}'_e)$ that we introduced in the biconvex formulation. Now, we present MICP-GCS:

\begin{align}
    \label{eq:objGCS}
    &\min \sum_{e=(i,j)\in E}l(\Tilde{z}_e, \Tilde{z}'_e) \\
    \intertext{subject to constraints \eqref{eq:depotFlowOut}, \eqref{eq:depotFlowIn}, \eqref{eq:targetFlowIn}, \eqref{eq:lxBiconvex}, \eqref{eq:lyBiconvex}, \eqref{eq:speedBiconvex}, \eqref{eq:coneBiconvex},}
    \label{eq:flowConserveGCS}
    &\sum_{e\in E^{in}_i}(z'_{e,t},y_e) = \sum_{e\in E^{out}_i}(z_{e,t},y_e), \;\; \forall \; i \in V_{tar}, \\
    \label{eq:perspectiveGCS}
    & (z_e,y_e) \in \Tilde{X}_i, \; (z'_e,y_e) \in \Tilde{X}_j, \;\; \forall \; e=(i,j) \in E. 
\end{align}

The MICP-GCS formulation differs from the biconvex formulation in the following ways: The constraints, \eqref{eq:flowConserveGCS} are obtained by combining \eqref{eq:flowConservation} with the additional constraints, $\sum_{e \in E^{in}_i}z'_{e,t} = \sum_{e \in E^{out}_i}z_{e,t}$ for each node $i \in V_{tar}$. These additional constraints \revision{ensure} that the time at which the agent visits each target $i$ will be equal to the time at which the agent departs from target $i$. \revision{Additionally}, \eqref{eq:trajConvex} which encapsulates \eqref{eq:nodeTimeWin}, \eqref{eq:defPx}, \eqref{eq:defPy}, as well as the bilinear constraints, \eqref{eq:defz,z'} are replaced by the constraints, \eqref{eq:perspectiveGCS}. These constraints \revision{require} that for each edge $e=(i,j) \in E$, $(z_e,y_e)$ and $(z'_e,y_e)$ \revision{lie} within the \emph{perspective}\footnote{The perspective of a compact, convex set $X \subset \mathbb{R}^n$ is defined as $\Tilde{X}:=\{(x,\lambda):\lambda \geq 0, x \in \lambda X\}$.} of the convex sets $X_i$ and $X_j$ respectively. This is a compact way of representing the set of constraints for all the edges $e=(i,j) \in E$ as shown below:

\begin{align}
    \label{eq:defz_et}
    & y_e\underline{t}_i \leq z_{e,t} \leq y_e\overline{t}_i, \\
    \label{eq:defz'_et}
    &y_e\underline{t}_j \leq z'_{e,t} \leq y_e\overline{t}_j, \\
    \label{eq:defz_ex}
    &z_{e,x} - v_{i,x}z_{e,t} - y_e(\underline{p}_{i,x} - \underline{t}_iv_{i,x}) = 0, \\
    \label{eq:defz_ey}
    &z_{e,y} - v_{i,y}z_{e,t} - y_e(\underline{p}_{i,y} - \underline{t}_iv_{i,y}) = 0, \\
    \label{eq:defz'_ex}
    &z'_{e,x} - v_{j,x}z'_{e,t} - y_e(\underline{p}_{j,x}-\underline{t}_jv_{j,x}) = 0, \\
    \label{eq:defz'_ey}
    &z'_{e,y} - v_{j,y}z'_{e,t} - y_e(\underline{p}_{j,y}-\underline{t}_jv_{j,y}) = 0. 
\end{align}

From an optimal solution to the MICP-GCS, we can recover an optimal agent tour for the biconvex formulation as follows:
For each node $i \in V$, find the optimal $(p_i,t_i)$, using the following equations: 

\begin{align}
    \label{eq:recoverPti}
    & (p_i,t_i) = \sum_{e \in E^{in}_i}z'_e, \;\; \forall \; i \in V \setminus \{s\}, \\
    \label{eq:recoverPts'}
    &(p_{s},t_{s}) = \sum_{e \in E^{out}_{s}}z_e.
\end{align}

\subsection{Proof of Validity}
In this section, we will show the correctness of the MICP-GCS formulation by proving the following theorem.

\begin{theorem}
    The optimal value of the MICP-GCS formulation is equal to the optimal value of the biconvex formulation for the MT-TSP. An optimal agent tour for the MT-TSP can be recovered from the solution of MICP-GCS by 
    choosing $(p_i,t_i) \; \forall \; i \in V$ as shown in \eqref{eq:recoverPti} and \eqref{eq:recoverPts'}.
\end{theorem}
\begin{proof}
    Let $E_{tour}$ be the set of all edges with $y_e=1$, obtained from a solution to either of the two formulations. The constraints in both the formulations \revision{require} that the edges in $E_{tour}$ \revision{form} a path that starts at $s$, visits all the target nodes once, and ends at $s'$, thereby forming an agent tour. For each edge $e \notin E_{tour}$, $z_e=z'_e=(0,0,0)$ for both the formulations. This is achieved by \eqref{eq:defz,z'} in the biconvex formulation, and by \eqref{eq:perspectiveGCS} in the MICP-GCS formulation. Consequently, the cost addends corresponding to these edges \revision{become} $l(\Tilde{z}_e,\Tilde{z}'_e)=0$ for both formulations. Now, consider each edge $e=(i,j) \in E_{tour}$. In the biconvex formulation, \eqref{eq:defz,z'} becomes $z_e=(p_i,t_i)$, $z'_e=(p_j,t_j)$. Additionally, \eqref{eq:trajConvex} requires $(p_i,t_i) \in X_i$, $(p_j,t_j) \in X_j$. These two requirements are the same as saying $z_e \in X_i$, $z'_e \in X_j$, and for any two adjacent edges $e=(i,j)$ and $f=(j,k)$ in the agent tour, $z'_{e}=z_{f}$. In the MICP-GCS formulation, \eqref{eq:perspectiveGCS} becomes $z_e \in X_i$, $z'_e \in X_j$, and the additional flow requirement in \eqref{eq:flowConserveGCS} ensures that for any two adjacent edges $e=(i,j)$ and $f=(j,k)$ in the agent tour, $z'_{e}=z_{f}$. Therefore, the cost addends corresponding to edges $e \in E_{tour}$ \revision{become} $l(\Tilde{z}_e, \Tilde{z}'_e)$ for both the formulations. Now, suppose that we have a solution to the MICP-GCS formulation. The $(p_i,t_i)$ corresponding to each $i \in V$ can then be obtained as shown in \eqref{eq:recoverPti} and \eqref{eq:recoverPts'} since \revision{they ensure} $(p_i,t_i)=z_e$ and $(p_j,t_j)=z'_e$ for each $e=(i,j) \in E$ with $y_e=1$.
\end{proof}

\section{Numerical Results}

\subsection{Test Settings and Instance Generation}
All the tests were run on a laptop with an Intel Core I7-7700HQ 2.80GHz CPU, and 16GB RAM. The implementation was in Python 3.11.6, and both the MICP as well as MICP-GCS formulations were solved using Gurobi 10.0.3 optimizer \cite{gurobi}. All the Gurobi parameters were set to their default values, except for \emph{TimeLimit}\footnote{Limits the total time expended (in seconds).}, which was set to 1800.

A total of 80 instances were generated, 20 each for 5, 10, 15, and 20 targets. The instances were defined by the number of targets $n$, a square area of fixed size $S=100$ units with corresponding diagonal length $R=\sqrt{2}S$, a fixed time-horizon $T=150$ secs, the depot location fixed at the center $(0,0)$ of the square, and finally, a set of randomly generated linear trajectories corresponding to the $n$ targets such that each target has a constant speed within $[0.5,1]$ $unit/sec$ and is confined within the square area.

For each instance, we ran experiments where we varied two additional test parameters: the maximum agent speed $v_{max}$, and the time-windows corresponding to the targets. $v_{max}$ was varied to be 4, 6, and 8 $unit/sec$, and the time-windows were varied to be of durations 25, 50, and 75 secs. The time-windows were selected such that a feasible solution could be found for all the generated instances, with all $v_{max}$ choices. To do this, we first randomly chose a sequence in which the agent visits the targets, and then found the quickest agent tour corresponding to that sequence by fixing $v_{max}$ at its lowest choice (which is 4 $unit/sec$). If the time for the tour was more than $T$, we tried another random sequence. Otherwise, we took the times when the agent visited each target, and defined time-windows that contained these times. Note that when varying time-windows, we ensured that the time-window of duration 25 lies within the time-window of duration 50, which then lies within the time-window of duration 75, for each target.

To evaluate the MICP and MICP-GCS formulations, we use \%~Gap, and runtime, which we will now explain. Given $v_{max}$, a time-window duration, and the formulation of choice, the solver is first run on all the 20 instances corresponding to a given number of targets. The optimality gap value from the solver for an instance is defined as $\frac{|z_P-z_D|}{|z_P|}\times100$, where $z_P$ is the primal (feasible) objective, and $z_D$ is the dual (lower-bound) objective. \%~Gap denotes the average of the smallest gap values output by the solver for all these instances, and runtime denotes the average of the run-times output by the solver for all these instances. 

\subsection{Varying the Time-Window Duration}

\begin{figure}[tb]
    \centering
    \includegraphics[width=\linewidth]{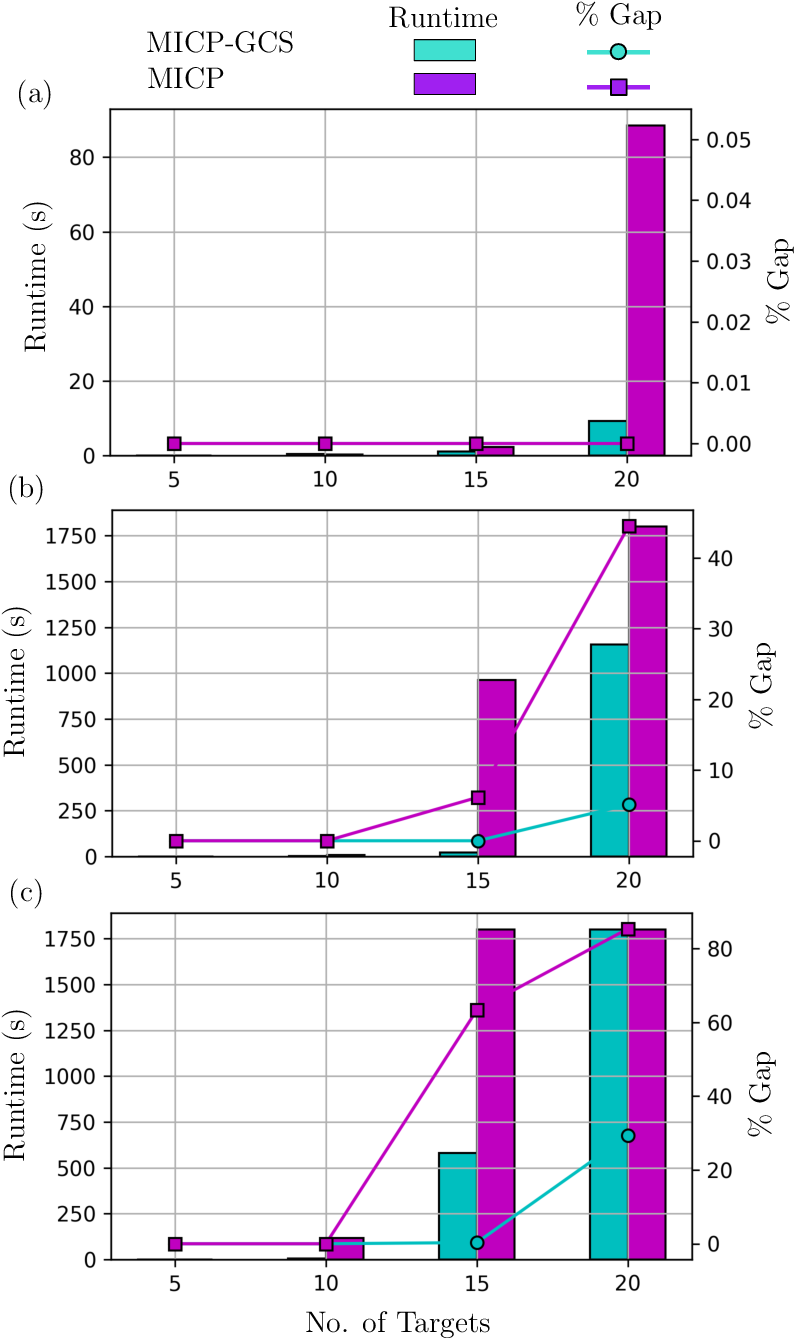}
    \caption{Numerical results comparing runtime and \% Gap of the MICP and MICP-GCS for a fixed $v_{max}$ of 4, and varying time-window durations of 25 (a), 50 (b), and 75 (c). MICP-GCS scales significantly better than the MICP, when increasing the time-window duration, and number of targets. This can be seen especially for 15 targets in (b) and 10 targets in (c) where it runs up to 2 orders of magnitude faster while providing the same or better \% Gap. Similarly, in the case of 15 targets in (c), and 20 targets in (b) and (c), MICP-GCS runs up to more than 1000 seconds faster, while providing a \% Gap improvement within a 40-60 range.} 
    \label{fig:exprTw}
\end{figure}

In this section, we consider the experiments where the time-window durations are varied. We do this by fixing $v_{max}$ at 4, and solving all the instances for the time-window durations (25, 50, and 75). The results for this experiment are illustrated in Fig.~\ref{fig:exprTw}, with (a), (b), and (c) corresponding to durations 25, 50, and 75 respectively. We observe that the problem becomes more challenging to solve for both the approaches as the number of targets increases. More importantly, this difficultly becomes more prominent as the time-window duration increases. The main advantage of MICP-GCS here is that it scales significantly better than the MICP against a larger number of targets and bigger time windows. We see this especially in the case of 15 targets where the \% Gap always fully converges for the MICP-GCS and its runtime increases noticeably only with the largest time-window of duration 75, as compared to the MICP whose \% Gap and runtime increases dramatically as the time-windows gets bigger. Note how the problem is challenging for both the approaches at 20 targets. However, we see that the \% Gap or the runtime is always significantly improved for MICP-GCS, for all time-window durations in this case.

\subsection{Varying the Agent Speed}

\begin{figure}[tb]
    \centering
    \includegraphics[width=\linewidth]{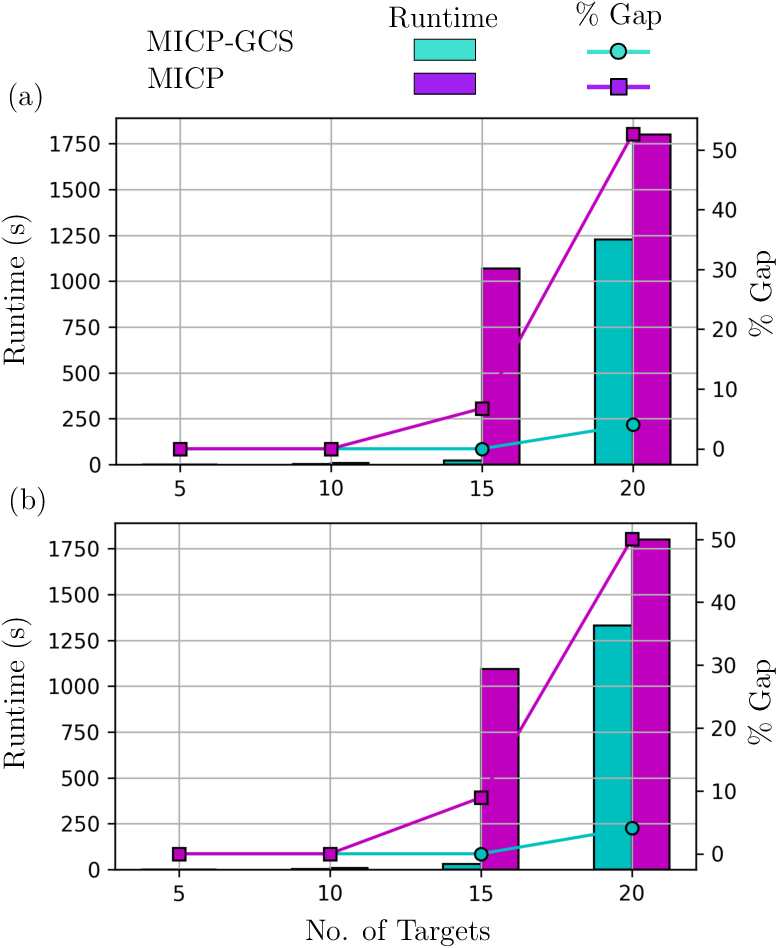}
    \caption{Numerical results comparing runtime and \% Gap of the MICP and MICP-GCS for a fixed time-window duration of 50, and varying $v_{max}$ choices of 6 (a), and 8 (b). The plots are very similar to the $v_{max}$ of 4 plot (Fig.~\ref{fig:exprTw} (b)). Hence, here too, MICP-GCS gives two orders of magnitude faster runtime with a \% Gap improvement of close to 10 for 15 targets. For 20 targets, MICP-GCS is still several hundreds of seconds faster, and gives a \% Gap improvement of around 45.}
    \label{fig:exprSpd}
    \vspace{-2mm}
\end{figure}

In this section, we consider the second set of experiments where $v_{max}$ is varied. We do this by fixing the time-window duration at 50, and varying $v_{max}$ to 6, and 8. The results for this experiment are illustrated in Fig.~\ref{fig:exprSpd}, with (a) and (b) corresponding to $v_{max}$ choices of 6 and 8 respectively. Note that Fig.~\ref{fig:exprTw} (b) gives the plot for when $v_{max}$ is 4. We observe that the plots look similar for all the three $v_{max}$ choices, with the runtime increasing slightly for both the approaches, and the \% Gap getting slightly larger for the MICP, as $v_{max}$ is increased. This small increase in difficulty is believed to stem mostly from the fact that the feasible search space is now larger, as the agent now has more choices of tours it can take. In all these plots we again observe that MICP-GCS vastly outperforms the MICP, both in terms of runtime as well as \% Gap, for 15 and 20 targets.

\subsection{Evaluating the Convex Relaxations}
In this section, we evaluate the lower-bounds to the MT-TSP obtained from the convex relaxations of both the MICP and MICP-GCS formulations. To do this, we use ratio, and runtime, which we will now explain. Given a $v_{max}$ value, a time-window duration, and the formulation of choice, the binary constraints are first relaxed, and the solver is run on all the 20 instances corresponding to a given number of targets. Ratio is then obtained by finding the ratio of the best bound output by the solver, and the best bound output by MICP-GCS previously when the binary constraints were not relaxed, for all these instances, and then finding the average of these values. Runtime is obtained by finding the average of the solver runtime for all these instances. Note that a higher ratio is always better as it shows that the convex relaxation provides tight lower-bounds comparable to the ones from MICP-GCS with binary constraints. The worst ratio achievable is 0, indicating a trivial lower-bound from the convex relaxation. The ratios and runtimes found are summarized in Table~\ref{tab:convexRatio}, and Table~\ref{tab:convexRt} respectively. 

\vspace{3mm}
\begin{table}[h]
    \centering
    \begin{adjustbox}{max width=\columnwidth,center} 
        \begin{tabular}{c|cccc} 
            \toprule
            Expr & 5 Tar & 10 Tar & 15 Tar & 20 Tar \\
            \midrule
            Tw25 & 0.95 & 0.82 & 0.77 & 0.68 \\
            Tw50 & 0.78 & 0.63 & 0.61 & 0.54 \\
            Tw75 & 0.66 & 0.56 & 0.48 & 0.54 \\
            Spd6 & 0.77 & 0.65 & 0.61 & 0.53 \\
            Spd8 & 0.77 & 0.65 & 0.61 & 0.54 \\
            \bottomrule
        \end{tabular}
    \end{adjustbox}
    \vspace{0.5mm}
    \caption{Numerical results presenting the ratios obtained from relaxed MICP-GCS for different experiment settings, and number of targets. We do not include the ratios for relaxed MICP here, since they were always the worst value of 0. The ratios get worse with more targets, and larger time-windows. Varying $v_{max}$ has negligible effect on the ratio.}
    \label{tab:convexRatio}
    \vspace{-5mm}
\end{table}

\begin{table}[h]
    \centering
    \begin{adjustbox}{max width=\columnwidth,center} 
        \begin{tabular}{c|cccc} 
            \toprule
            Expr & 5 Tar & 10 Tar & 15 Tar & 20 Tar \\
            \midrule
            Tw25 & 0.01 (0.0) & 0.03 (0.02) & 0.06 (0.04) & 0.12 (0.02) \\
            Tw50 & 0.01 (0.0) & 0.02 (0.02) & 0.04 (0.04) & 0.08 (0.02) \\
            Tw75 & 0.01 (0.0) & 0.01 (0.02) & 0.04 (0.06) & 0.07 (0.03) \\
            Spd6 & 0.0 (0.0) & 0.01 (0.02) & 0.03 (0.04) & 0.08 (0.02) \\
            Spd8 & 0.01 (0.0) & 0.01 (0.02) & 0.03 (0.04) & 0.08 (0.02) \\
            \bottomrule
        \end{tabular}
    \end{adjustbox}
    \vspace{0.5mm}
    \caption{Numerical results presenting the runtimes for both the relaxed MICP-GCS, and relaxed MICP (in parentheses) for different experiment settings, and number of targets. The runtimes increase slightly with more targets, but are overall negligible for both formulations.}
    \label{tab:convexRt}
    \vspace{-2mm}
\end{table}

 In both the tables, the column {\tt Expr} represents the experiment settings, and specifies the various choices of $v_{max}$ and time-window duration. Here, {\tt Tw25}, {\tt Tw50}, and {\tt Tw75} represent the same experiment settings used for Fig.~\ref{fig:exprTw}, where $v_{max}$ was set at 4, and the time-window duration was varied to be 25, 50, and 75 respectively. Similarly, {\tt Spd6} and {\tt Spd8} represent the experiment settings used for Fig.~\ref{fig:exprSpd}, where the time-window duration was fixed at 50, and $v_{max}$ was varied to be 6 and 8. 
 
 In Table~\ref{tab:convexRatio}, we only provide the ratios corresponding to the convex relaxation of MICP-GCS. This is because when relaxed, the MICP always gave the worst bound of 0. This is to be expected, since this formulation relies heavily on big-$M$ constraints. Observe how the lower-bounds from relaxed MICP-GCS is affected by the number of targets and the time-window durations, but not from varying choices of $v_{max}$. This is consistent with our previous observations. Although the lower-bounds to the MT-TSP are somewhat crude here, especially with larger number of targets and bigger time-windows, the main advantage of relaxing MICP-GCS comes from its negligible runtimes as seen in Table~\ref{tab:convexRt}. Observe that the runtimes from relaxed MICP-GCS are similar to the ones from the relaxed MICP (values within parentheses), but provides significantly stronger lower-bounds to the MT-TSP.
 
\section{Conclusion and Future Work}
In this paper, we introduced a Mixed Integer Conic Program based on the graph of convex sets (MICP-GCS), that finds the optimum to a special case of the Moving-Target Traveling Salesman Problem where targets move along lines with constant speeds. We proved the validity of this new formulation, and presented numerical results to corroborate its performance. We showed how our MICP-GCS outperforms the current state-of-the-art MICP across various experiments, and also how the MICP-GCS has a much stronger convex relaxation than the baseline MICP. For future work, \revision{we plan on investigating the effectiveness of MICP-GCS, when extended to handle multiple agents, and piecewise-convex target trajectories}

\addtolength{\textheight}{-12cm}   





\section*{Acknowledgment}
This material is based upon work supported by the National Science Foundation under Grant 2120219 and Grant 2120529. Any opinions, findings, and conclusions or recommendations expressed in this material are those of the author(s) and do not necessarily reflect the views of the National Science Foundation. 


\bibliographystyle{IEEEtran}
\bibliography{references}

\end{document}